\documentclass{article}
\usepackage{ijcai25}

\usepackage{times}
\usepackage{soul}
\usepackage{url}
\usepackage[hidelinks]{hyperref}
\usepackage[utf8]{inputenc}
\usepackage[small]{caption}
\usepackage{graphicx}
\usepackage{amsmath}
\usepackage{amsthm}
\usepackage{booktabs}
\usepackage{algorithm}
\usepackage{algorithmic}
\usepackage[switch]{lineno}

\usepackage{amsfonts}
\usepackage{multirow}
\usepackage[switch]{lineno}
\usepackage{subfig}
\usepackage[labelformat=simple]{caption}
\usepackage{pdfpages}
\usepackage{booktabs} 

\newtheorem{proposition}{Proposition}
\newtheorem{corollary}{Corollary}

\title{Block Circulant Adapter for Large Language Models}

\author{
Xinyu Ding \and
Meiqi Wang \and
Siyu Liao \And
Zhongfeng Wang \\
\affiliations
Sun Yat-sen University \\
\emails
dingbai1357718507@gmail.com,
wangmq53@mail.sysu.edu.cn,
liaocs2008@gmail.com,
wangzf83@mail.sysu.edu.cn
}

\begin{document}
\maketitle

\begin{abstract}
Fine-tuning large language models (LLMs) is difficult due to their  huge model size. Recent Fourier domain-based methods show potential for reducing fine-tuning costs. 
We propose a block circulant matrix-based fine-tuning method with a stable training heuristic to leverage the properties of circulant matrices and one-dimensional Fourier transforms to reduce storage and computation costs.
Experiments show that our method uses $14\times$ less number of parameters than VeRA, $16\times$ smaller than LoRA and $32\times$ less FLOPs than FourierFT,  while maintaining close or better task performance.
Our approach presents a promising way in frequency domain to fine-tune large models on downstream tasks.
\end{abstract}

\section{Introduction}

Large language models (LLMs) have been applied to serve as foundation models for many applications \cite{cheng2024prompt} because of their outstanding performance in various tasks.
Different from many traditional deep learning applications, these models are trained in an unsupervised fashion on large scale of data.
The huge data volume also indirectly forces LLMs to perform unsupervised training, since collecting human labels can be expensive and slow. 
Starting from the classical BERT model \cite{DBLP:conf/naacl/DevlinCLT19}, LLMs develops into GPT model \cite{GPT}, which inspires many modern LLMs like the LLaMA model \cite{llama}. 

Given a well trained LLM, it is not directly applicable to downstream tasks.
As a task is often highly customized (e.g., write a story based on some keywords), there is some need to further fine-tune the model to better fit the application. 
In general, there are three ways to fine-tune a LLM, i.e., full fine-tuning, partially fine-tuning and prompt tuning. 
Full fine-tuning means training all parameters of the LLM on downstream task data, but it can be challenging due to the huge computation resource cost.  
Partially fine-tuning is training a small part of parameters, which takes much less memory, computation time, and power consumption. 
Prompt tuning is to add more description (e.g., some task examples) in input so that LLM can learn during inference time. 
Such capability is also called in context learning. 
However, prompt tuning can be difficult as it is unclear how LLM understands the prompt. 
For example, changing the task examples amount or order can significantly affect the model performance \cite{lu2022fantastically}. 
Given the high cost of full fine-tuning and unknown mechanism underlying prompt tuning, this paper studies efficient methods to partially fine-tune LLMs.

It should be noted that partially finetuning method is also called parameter-efficient fine-tuning (PEFT) method in the literature.
Many PEFT methods tend to freeze LLM parameters and train on added and task dependent parameters, which are called adapters \cite{adapter}. 
For example, low rank adapter (LoRA) by \cite{hu2021lora} factorizes weight matrices into trainable low rank components.
Ladder side tuning (LST) method \cite{lst} inserts flexible and parameterized modules that help reduce memory consumption. 
The state-of-the-art Fourier domain based fine-tuning (FourierFT) by \cite{DBLP:conf/icml/GaoWCLWC024} utilizes highly sparse parameter matrix in Fourier domain to construct weight matrix. 

Adapters can be categorized into mergeable and non-mergable adapters. 
Mergeable adapters can be merged into the LLM after training, since their parameters are in the same shape as LLM parameters. 
The merging process is often an summation operation that adds up original LLM parameters and adapter parameters, e.g., like the LoRA method. 
As a result, these adapters do not incur extra inference cost after training.
Other adapters like LST method does not have this advantage.
Instead, without the parameter shape restriction, they are more flexible in the adapter architecture design.

In this paper, we focus on the adapter design that can be merged into LLM after training. 
Given the recent success of Fourier domain based method, the result of significantly small number of training parameters is promising. 
We notice that the FourierFT method uses 2D FFT operations which can be very expensive in terms of computation cost. 
In contrast, previous success of circulant structure in deep learning \cite{cheng2015exploration} with 1D FFT operation seems to be able to further reduce the computation cost of Fourier domain based fine-tuning method. 
However, we find that circulant structure based training method can diverge in LLMs.
Through theoretical analysis and empirical experiments, we propose block circulant matrix based  adapter with a special training heuristic to help model converge in practice. 
Overall, the \textbf{contribution} of this paper can be summarized as follow:

We propose the Block Circulant Adapter (BCA), a novel and efficient parameter-efficient fine-tuning approach for fine-tuning LLMs. BCA leverages the structure of block circulant matrices and FFT operations to significantly reduce both storage consumption and computation  complexity.  
This approach not only ensures stable and convergent training by mitigating the gradient explosion risk associated with block circulant matrix-based linear layers but also achieves a "win-win" scenario in terms of efficiency and scalability for LLMs.

We are the first to theoretically analyze and demonstrate that the gradient explosion risk associated with block circulant matrix-based linear layers can be effectively mitigated by a heuristic of learning rate adjustment, ensuring stable and convergent training processes for BCA, which is crucial for practical application of LLM adapters.

Extensive experiments on standard NLP tasks and datasets substantiate the effectiveness of our BCA method. We demonstrate that BCA not only matches the performance of prior works like LoRA and FourierFT but also achieves this with substantially lower computational costs and storage costs, highlighting the advantages of our approach for LLM fine-tuning.

\section{Related Work}
In the domain of LLMs, the evolution of Parameter-Efficient Fine-Tuning methods has marked a significant shift towards enhancing model performance without the computational burden of full-parameter tuning. 
Besides the capability of mergeable or not into LLMs, PEFT methods are also classified into additive, selective, reparameterized, and hybrid fine-tuning approaches in the literature \cite{Xu2023ParameterEfficientFM}. 
Additive fine-tuning methods, such as adapters \cite{adapter} and soft prompts \cite{Wang2023MultitaskPT}, introduce additional trainable parameters into the model architecture, allowing for efficient task-specific tuning without modifying the pre-trained parameters. 
Selective fine-tuning methods focus on updating a subset of the model's parameters, thereby reducing the computational overhead \cite{perifanis2024sftc}. 
Reparameterized fine-tuning methods, such as low-rank adapter \cite{hu2021lora} and its derivatives \cite{dettmers2024qlora}, involve constructing a low-dimensional reparameterization of the original model parameters for training, which is then transformed back to maintain inference speed. 
Hybrid fine-tuning methods \cite{zhou2024autopeft} combine the advantages of different PEFT approaches to build a unified model that balances efficiency and performance.

In parallel, the exploration of the Fourier domain for fine-tuning LLMs has opened new avenues, especially with studies indicating that LLMs leverage Fourier domain features in certain tasks \cite{zhou2024pre}. 
For example, FourierFT by \cite{DBLP:conf/icml/GaoWCLWC024} significantly reduces the number of trainable parameters while maintaining or even improving performance across various tasks. 
Existing Fourier-based tuning methods, have not yet resolved the issue of memory consumption due to the necessity of weight matrix generation through 2D Fast Fourier Transform. 
This limitation has spurred the need for a more efficient method that can leverage the Fourier domain without incurring excessive memory and computation costs.

However, Fourier domain based training method can be traced back to compressing computer vision  models \cite{cheng2015exploration,ding2017circnn}. 
The circulant structure is utilized due to its connection to FFT operation.
It is found that circulant structure falls into the family of low displacemnt rank matrix which is proven applicable in deep learning \cite{zhao2017theoretical}.
These methods are further generalized into displacement rank learning \cite{thomas2018learning,zhao2017theoretical} of structure matrices.
Later, it gradually develops into the design of structure matrices to enable powerful feature  representations \cite{dao2022monarch}.

\begin{figure*}[bt]
\centering
\includegraphics[width=0.85\textwidth]{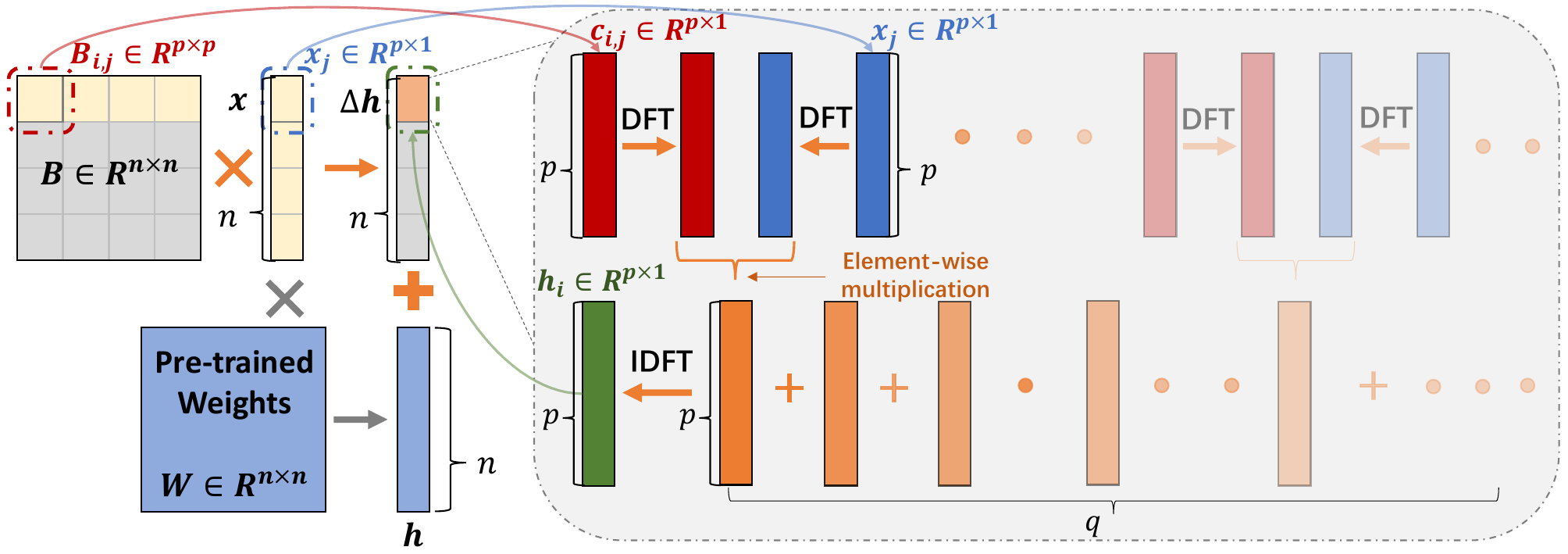}
\caption{
Illustration of block circulant adapter.
$\mathbf{h}$ is the output from pre-trained weight matrix $\mathbf{W}$.
$\Delta\mathbf{h}$ is the output from fine-tuned weight change matrix $\Delta\mathbf{W}$ which in this case is the block circulant matrix $\mathbf{B}$.
The summation of $\mathbf{h}$ and $\Delta\mathbf{h}$ form the output of the fine-tuned model.
After training, the block circulant matrix can be directly added to pre-trained weight matrix such that no extra inference cost is incurred. 
}
\label{fig:block_circ_forward}
\end{figure*}

\section{Preliminary}
In this section, we briefly introduce the concept of circulant matrix and block circulant matrix.
The related matrix vector multiplication algorithm is also shown via using fast fourier transform (FFT) algorithm. 
It should be noted that the corresponding back-propagation \cite{lecun1988theoretical} formulation is skipped in this paper, since auto-differentiation \cite{paszke2017automatic} has been widely adopted in modern deep learning frameworks.

\subsection{Circulant Matrix}
Given a vector $\mathbf{c}=\{c_i\}_{i=0}^{n-1}\in\mathbb{R}^{n\times 1}$, the circulant matrix $circ(\mathbf{c})\in\mathbb{R}^{n\times n}$ can be determined by:
\begin{equation}
\label{eq:circ}
circ(\mathbf{c}) = 
\begin{bmatrix}
    c_0 & c_{n-1}  & \dots  & c_1 \\
    c_1 & \ddots  & \ddots & \vdots \\
    \vdots & \ddots & \ddots & c_{n-1} \\
    c_{n-1} & \dots & c_1  & c_0
\end{bmatrix}.
\end{equation}
For matrix vector multiplication, there exist a fast multiplication algorithm \cite{oppenheim1999discrete} for circulant matrix utilizing fast fourier transform:
\begin{equation}
\label{eq:circ_matvec}
circ(\mathbf{c})\mathbf{x}
=
\text{IFFT}(\text{FFT}(\mathbf{c}) \circ \text{FFT}(\mathbf{x})),
\end{equation}
where $\mathbf{x}\in\mathbb{R}^{n\times 1}$ is an input vector, IFFT is the inverse fast frourier transform, and $\circ$ is the element-wise product, respectively.
In addition, circulant matrix can also be written as a matrix polynomial:
\begin{align}
\label{eq:polynomial}
\begin{split}
\mathbf{P}
&=
\begin{bmatrix}
    0 & 0 & \dots & 0  & 1 \\
    1 & 0 & \ddots &   & 0 \\
    0 & \ddots & \ddots & \ddots & \vdots \\
    \vdots & & \ddots & \ddots & 0 \\
    0 & \dots & 0 & 1  & 0
\end{bmatrix}
,
\\
circ(\mathbf{c})
&=
c_0\mathbf{I}
+
c_1\mathbf{P}
+
\dots
+
c_{n-1}\mathbf{P}^{n-1},
\end{split}
\end{align}
where $\mathbf{P}\in\mathbb{R}^{n\times n}$ is a cyclic permutation matrix.

\subsection{Block Circulant Matrix}
A block circulant matrix is a block matrix with each block being a circulant matrix. 
For the simplicity of implementation, the block circulant matrix in deep learning community is often an equally partitioned matrix \cite{ding2017circnn}. 
More specifically, let $\mathbf{B}\in\mathbb{R}^{n\times n}$ be a block circulant matrix with partition size $p$ such that $n/p=q$.
Then, each row and each column contain $q$ submatrices in shape $p\times p$ that is denoted by $\mathbf{B}_{i,j}$, where $i$ and $j$ are integers from $0$ to $q-1$, respectively. 
According to Eq. (\ref{eq:circ}), assume that each submatrix $\mathbf{B}_{i,j}$ is defined by the corresponding vector $\mathbf{c}_{i,j}\in\mathbb{R}^{p\times 1}$. 
Similarly, let vector $\mathbf{x}$ be partitioned into $q$ subvectors with each subvector $\mathbf{x}_j\in\mathbb{R}^{p\times 1}$.
Denote the matrix vector multiplication result by column vector $\mathbf{h}\in\mathbb{R}^{n\times 1}$.
Apply the same partition into $\mathbf{h}$ such that $\mathbf{h}_i\in\mathbb{R}^{p\times 1}$.
The block matrix vector multiplication can be then written as:
\begin{align}
\label{eq:blockcirc_matvec}
\begin{split}
\mathbf{h} 
&=
\mathbf{B}\mathbf{x}
= \{\mathbf{h}_i\}_{i=0}^{p-1}
,
\\
\mathbf{h}_i
&=
\sum_{j=0}^{q-1}
\mathbf{B}_{i,j}\mathbf{x}_j
\\
&=
\sum_{j=0}^{q-1}
\text{IFFT}(\text{FFT}(\mathbf{c}_{i,j}) \circ \text{FFT}(\mathbf{x}_j))
\\
&=
\text{IFFT}(\sum_{j=0}^{q-1}
\text{FFT}(\mathbf{c}_{i,j}) \circ \text{FFT}(\mathbf{x}_j)),
\end{split}
\end{align}
where the IFFT computation is combined into single computation for each $\mathbf{h}_i$ \cite{liao2019circconv}.

It should be noted that block circulant matrix is a more flexible and general representation of circulant structure. 
When $p=n$, there is only one block, and block circulant matrix is a single circulant matrix. 
When $p=1$, block circulant matrix becomes a general dense matrix.
Therefore, block circulant matrix connects circulant with dense matrix by adjusting the partition size $p$.

\begin{figure*}[htp]
\centering
\includegraphics[width=0.85\textwidth]{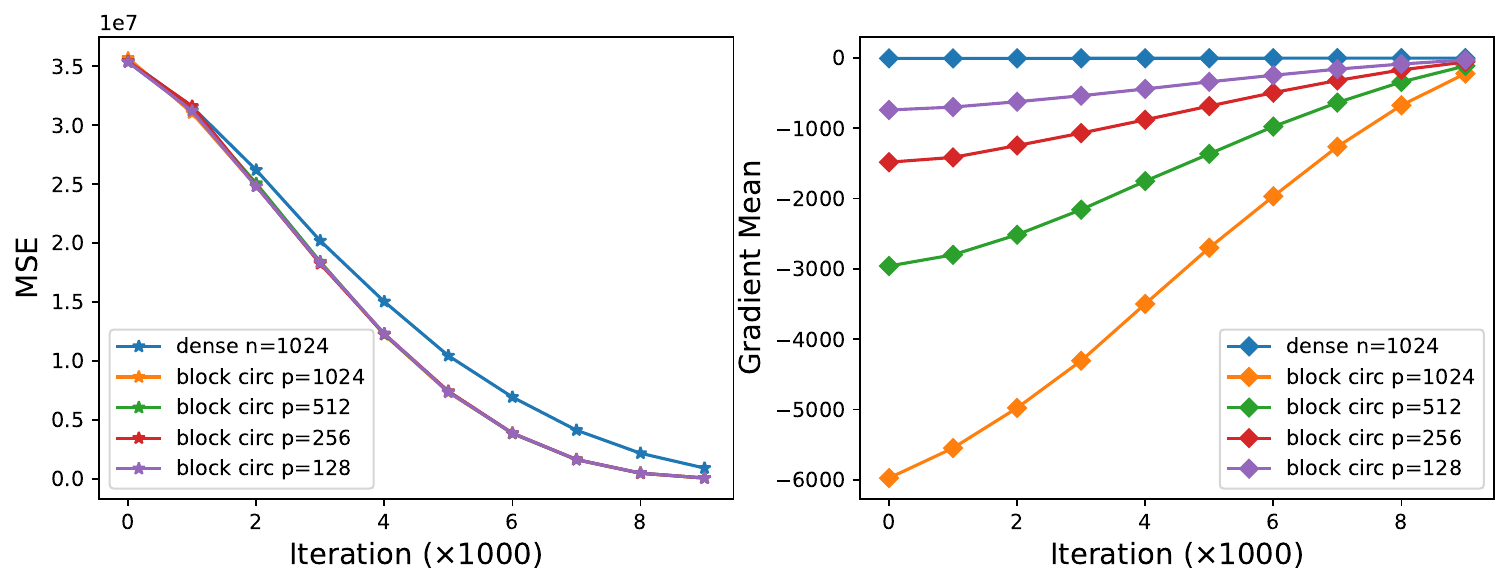}
\caption{
Gradient of single layer neural network with $n=1024$ for dense matrix and $p\in\{128, 256, 512, 1024\}$ for block circulant matrix.
Dense matrix can be seen as a special case of block circulant matrix with block size $p=1$.
When $p=n$, the block circulant matrix becomes a circulant matrix. 
MSE curve shows all simulated training processes converge. 
Gradient mean value curve demonstrates that gradient of block circulant matrix is proportional to block size setting $p$.
}
\label{fig:least_squre}
\end{figure*}

\section{Method}
There are many explorations of applying circulant structure to deep learning model  \cite{cheng2015exploration,DBLP:journals/corr/MoczulskiDAF15,ding2017circnn,thomas2018learning,liao2019circconv}.
In practice, it is found that training block circulant structure can occasionally diverge, especially when $p$ or $n$ becomes large. 
To the best of our knowledge, we for the first time show the reason why circulant structure can diverge sometimes with the first order derivative based optimizing approaches, which are widely adopted in the community.

In this section, we start with theoretical proof of the gradient explosion risk of block circulant matrix.
Next, we simulate training block circulant matrix on a single layer neural network.
It can be empirically observed that block circulant matrix results in much larger gradients than dense matrix. 
Last, we provide a simple but effective solution to ensure a stable training process for learning block circulant matrix. 

\subsection{Derivative of Circulant Structure}
Modern deep learning model optimization methods mainly focus on the first order derivative.
It can be proved that the first order derivative of block circulant matrix can explode in the sense that the corresponding gradient value is proportional to $p$. 
However, different from the gradient explosion effect like in recurrent neural network \cite{bengio1994learning}, the large gradient of block circulant matrix does not affect the back-propagation process for updating other neural network layers. 

\begin{proposition}
Given a vector $\mathbf{c}\in\mathbb{R}^{n\times 1}$ and an input vector $\mathbf{x}\in\mathbb{R}^{n\times 1}$, 
let $\mathbf{h}=\text{circ}(\mathbf{c})\mathbf{x}$ and $f(\mathbf{h}):\mathbb{R}^{n\times 1}\rightarrow \mathbb{R}$ be a differentiable function. 
The first order derivative of $\mathbf{c}$ has a bounded bilinear form. 
\end{proposition}

\begin{proof}
According to Eq. (\ref{eq:polynomial}), it can be noticed that the Jacobian matrix $\mathbf{J}$ for the linear mapping from $\mathbf{c}$ to $\mathbf{h}$ has the form $[\mathbf{I}, \mathbf{P}, \dots, \mathbf{P}^{n-1}]\mathbf{x}$. Next, we apply chain rule to calculate the first order  derivative as following:
\begin{align}
\begin{split}
\nabla f(\mathbf{c})
&=
\nabla f(\mathbf{h})^T \mathbf{J}
\\
&=
\nabla f(\mathbf{h})^T 
[\mathbf{I}, \mathbf{P}, \dots, \mathbf{P}^{n-1}]\mathbf{x},
\\
\nabla f(\mathbf{c}_i)
&= 
\nabla f(\mathbf{h})^T 
\mathbf{P}^i
\mathbf{x}
\\
&\leq 
||\nabla f(\mathbf{h})||\times||\mathbf{P}^{i}\mathbf{x}||
\\
&\leq 
||\nabla f(\mathbf{h})||\times ||\mathbf{x}||
,
\end{split}
\end{align}
where $||\cdot||$ is a norm function.
It can be seen that $\nabla f(\mathbf{c}_i)$ is a bounded bilinear mapping from $\nabla f(\mathbf{h})\times  \mathbf{x}$ to $\mathbf{c}_i$. 
\end{proof}

\begin{proposition}
Given a general matrix $\mathbf{A}\in\mathbb{R}^{n\times n}$, 
let $\mathbf{g}=\mathbf{A}\mathbf{x}$.
For the same differentiable function $f$,
$\min \{\nabla f(\mathbf{c})\} \geq n \times \min\{\nabla f(\mathbf{A})\}$.
\end{proposition}
\begin{proof}
Note that the derivative of $\mathbf{g}$ and $\mathbf{h}$ are independent of $\mathbf{A}$ and $\mathbf{c}$, thereby $\nabla f(\mathbf{g})=\nabla f(\mathbf{h})$. 
Then, we can compute the derivative of $\mathbf{A}$ by:
\begin{align}
\begin{split}
\nabla f(\mathbf{A})
&=
\nabla f(\mathbf{h}) \mathbf{x}^T
,
\\
\nabla f(\mathbf{A}_{j,i})
&=
\nabla f(\mathbf{h}_j)
\mathbf{x}_i
.
\end{split}
\end{align}
In the case of circulant vector based weight parameters, we have the result:
\begin{align}
\begin{split}
\nabla f(\mathbf{c}_i)
&= 
\nabla f(\mathbf{h})^T 
\mathbf{P}^i
\mathbf{x}
=
\sum_{j=0}^{n-1}
\nabla f(\mathbf{h}_j)
\mathbf{x}_{j-i}
,
\end{split}
\end{align}
where for $j-i < 0$, it refers to the index of $j-i+n$. 
Therefore, it can be seen that:
\begin{align}
\begin{split}
&\min \{\nabla f(\mathbf{c})\}
=
\min_i \{\nabla f(\mathbf{c}_i)\}
\\
&= 
\min_i \{\sum_{j=0}^{n-1} \nabla f(\mathbf{h}_j)
\mathbf{x}_{j-i}\}
\\
&\geq
\min_i \{n\times \{ \min_j \nabla f(\mathbf{h}_j)\mathbf{x}_{j-i}\}\}
\\
&\geq 
n\times \min_i \{ \min_j \{ \nabla f(\mathbf{h}_j)\mathbf{x}_{j-i}\}\}
\\
&\geq 
n\times \min_i \{ \min_j \{ \min \nabla f(\mathbf{A})\}\}
\\
&\geq 
n\times \min \{\nabla f(\mathbf{A})\}
.
\end{split}
\end{align}
\end{proof}

\begin{proposition}
\label{prop:blockcirculant}
Given a block circulant matrix $\mathbf{B}$, for any submatrix $\mathbf{B}_{i,j}$, the corresponding vector $\mathbf{c}_{i,j}$ satisfies $\min \{\nabla f(\mathbf{c}_{i,j})\} \geq p\times \min \{ \nabla f(\mathbf{A}) \}$.
\end{proposition}
\begin{proof}
According to Eq. (\ref{eq:blockcirc_matvec}), we can calculate derivative:
\begin{equation}
\nabla f(\mathbf{c}_{i,j})
=
\nabla f(\mathbf{h}_i)^T
[\mathbf{I}, \mathbf{Q}, \dots, \mathbf{Q}^{p-1}]\mathbf{x}_j,
\end{equation}
where $\mathbf{Q}$ is in the similar form as $\mathbf{P}$ but in shape $p\times p$. 
Thus, for each circulant submatrix, it follows from the aforementioned result of circulant matrix derivative:
\begin{equation}
    \min \{\nabla f(\mathbf{c}_{i,j})\} \geq p\times \min \{ \nabla f(\mathbf{A}) \} .
\end{equation}
\end{proof}

\begin{corollary}
Compared with gradients of dense matrix based linear layer, circulant matrix based linear layer gradient explodes by $n$, and block circulant matrix based linear layer explodes by $p$.
\end{corollary}

\subsection{Single Layer Neural Network Simulation}
Besides theoretical proof, we empirically simulate the gradient exploding phenomenon of circulant structure by training on a single linear layer neural network. 
Given a random matrix $\mathbf{W}\in\mathbb{R}^{n\times n}$, let $\mathbf{y}=\mathbf{W}\mathbf{x}+\mathbf{\epsilon}$ with $\mathbf{\epsilon}\sim \mathcal{N}(\mathbf{0}, \mathbf{I})$. 
The training data $(\mathbf{x},\mathbf{y})$ can be randomly generated on-the-fly. 
For dense matrix based linear layer, we train $\hat{\mathbf{W}}$ with prediction $\hat{\mathbf{y}}=\hat{\mathbf{W}}\mathbf{x}$ by minimizing the mean square error (MSE) between $\mathbf{y}$ and $\hat{\mathbf{y}}$:
\begin{equation}
    \text{MSE}(\hat{\mathbf{y}}, \mathbf{y}) = || \hat{\mathbf{y}} - \mathbf{y} ||^2.
\end{equation}
For block circulant matrix based linear layer, we train $\hat{\mathbf{B}}$ according to Eq. (\ref{eq:blockcirc_matvec}), where there are $q\times q$ submatrices that are circulant matrices.

In practice, we set batch size 32, training iteration 10000 and learning rate 0.1 for Adadelta optimizer \cite{zeiler2012adadelta}. 
For both dense matrix and block circulant matrix based linear layer, MSE and gradient mean values are reported and compared in Fig. \ref{fig:least_squre}. 
It can be seen that both dense matrix and block circulant matrix based linear layers converge.
But block circulant matrix based linear layer training ends with smaller and better MSE. 
Such result may be caused by the circulant structure that may work as regularization. 

As shown in Fig. \ref{fig:least_squre}, the gradient curve of block circulant matrix starts with large gradient values and decreases along with the increase of training iterations.
This is caused by the convergence of learning block circulant matrix.  
It can also be noticed that gradients of different block circulant matrices are linearly proportional to block size settings. 
The larger block size results in larger gradient value, and corresponding gradient is around $p$ times larger than dense matrix. 
When $p=1024$, the block circulant matrix results in a single block, and it becomes a circulant matrix that has the largest gradient value. 

\subsection{Learning Heuristic}
We propose using block circulant matrix based linear layer as adapter for fine-tuning large language models. 
More specifically, for a large weight matrix $\mathbf{W}$, the adapter learns the weight change matrix $\Delta \mathbf{W}$ by letting $\mathbf{B}=\Delta \mathbf{W}$. 
Given input $\mathbf{x}$, the forward process now becomes:
\begin{equation}
    \mathbf{h} + \Delta\mathbf{h} 
    = 
    \mathbf{W}\mathbf{x} + \Delta\mathbf{W}\mathbf{x} 
    = 
    \mathbf{W}\mathbf{x}  + \mathbf{B}\mathbf{x}
    .
\end{equation}
In this way, the storage complexity is linear, as block circulant matrix can be constructed from corresponding vectors. 
The computation complexity is loglinear, since block circulant matrix vector multipliation can be executed using FFT operations which has loglinear computation complexity. 
Fig. \ref{fig:block_circ_forward} visualizes the forward process. 
It can be seen that all computations are performed on vectors, thereby effectively reducing the training complexity especially compared with full fine-tuing on $\Delta\mathbf{W}$.

Besides, according to our theoretical proof and empircal observation, the gradient explosion risk of block circulant matrix can cause model to diverge occasionally. 
To achieve a stable training process, we propose to factor down the learning rate $\alpha$ by block size $p$:
\begin{equation}
\label{eq:lr}
    \alpha \gets \alpha / p. 
\end{equation}
This heuristic solution aims at updating weight parameters of each $\mathbf{c}_{i,j}$ inside $\mathbf{B}$ with gradient values as large as in general dense matrix. 
In this way, at each training step, the updated block circulant matrix does not change significantly, thereby resulting in less output change for next iteration. 
It turns out that this solution can effectively ensure a successful training of the block circulant adapter.

\begin{table*}[!tb]
\centering
\setlength{\tabcolsep}{8.5pt}
\resizebox{0.75\textwidth}{!}{
\begin{tabular}{c l|c c c c c c c}
\hline
\multicolumn{2}{c|}{Method} &  SST-2 & MRPC & CoLA & QNLI & RTE & STS-B & Avg. \\
\hline
\multirow{6}{*}{\rotatebox{90}{BASE}} & FF & $94.8$ & $90.2$ & $63.6$ & $92.8$ & $78.7$ & $91.2$ & 85.2 \\
~ & LoRA  & $95.1_{\pm0.2}$ & $89.7_{\pm0.7}$ & $63.4_{\pm1.2}$ & $93.3_{\pm0.3}$ & $78.4_{\pm0.8}$ & $91.5_{\pm0.2}$ & 85.2 \\
~ & VeRA &$ 94.6_{\pm0.1} $ &$ 89.5_{\pm0.5}$ &$ 65.8_{\pm0.8}$ &$ 91.8_{\pm0.2}$ &$ 78.7_{\pm0.7}$ &$ 90.7_{\pm0.2}$ & 85.2 \\
~ & FourierFT &$ 94.2_{\pm0.3} $ &$ 90.0_{\pm0.8}$ &$ 63.8_{\pm1.6}$ &$ 92.2_{\pm0.1}$ &$ 79.1_{\pm0.5}$ &$ 90.8_{\pm0.2}$ & 85.0 \\                
~ & $\text{Ours}_{p=256}$  &$94.7_{\pm0.5} $ &$ 89.5_{\pm0.7}$ &$ 62.1_{\pm0.6}$ &$ 91.8_{\pm0.2}$ &$ 80.1_{\pm1.4}$ &$ 90.8_{\pm0.1}$ & 84.8\\
~ & $\text{Ours}_{p=768}$  &$93.8_{\pm0.1} $ &$ 89.0_{\pm0.6}$ &$ 64.4_{\pm1.6}$ &$ 92.0_{\pm0.2}$ &$ 79.8_{\pm0.9}$ &$ 90.3_{\pm0.3}$ & 84.9 \\
\hline
\multirow{6}{*}{\rotatebox{90}{LARGE}} & FF & $96.4$ & $90.9$ & $68.0$ & $94.7$ & $86.6$ & $92.4$ & 88.2 \\
~ & LoRA & $96.2_{\pm0.5}$ & $90.2_{\pm1.0}$ & $68.2_{\pm1.9}$ & $94.8_{\pm0.3}$ & $85.2_{\pm1.1}$ & $92.3_{\pm0.5}$ & 87.8 \\
~ & VeRA & $96.1_{\pm0.1}$ & $90.9_{\pm0.7}$ & $68.0_{\pm0.8}$ & $94.4_{\pm0.2}$ & $85.9_{\pm0.7}$ & $91.7_{\pm0.8}$ & 87.8 \\
~ & FourierFT & $96.0_{\pm0.2}$ & $90.9_{\pm0.3}$ & $67.1_{\pm1.4}$ & $94.4_{\pm0.4}$ & $87.4_{\pm1.6}$ & $91.9_{\pm0.4}$ & 88.0 \\
~ & $\text{Ours}_{p=512}$ &  $96.1_{\pm0.3}$ & $90.7_{\pm1.0}$ & $67.8_{\pm1.0}$ & $94.1_{\pm0.1}$ & $88.1_{\pm1.2}$ & $91.7_{\pm0.3}$ & 88.1\\
~ & $\text{Ours}_{p=1024}$ &  $95.8_{\pm0.2}$ & $89.7_{\pm0.7}$ & $66.7_{\pm0.8}$ & $94.2_{\pm0.1}$ & $85.6_{\pm1.2}$ & $91.4_{\pm0.4}$ & 87.2 \\
\hline
\end{tabular}
}
\caption{Fine-tuning RoBERTa base and large model on GLUE. Larger metric value means better model performance.}
\label{tbl:rob_glue}
\end{table*}

\begin{figure}[!t]
\centering
\subfloat[\scriptsize Partition size $p=1024$]{
\includegraphics[width=0.44\textwidth]{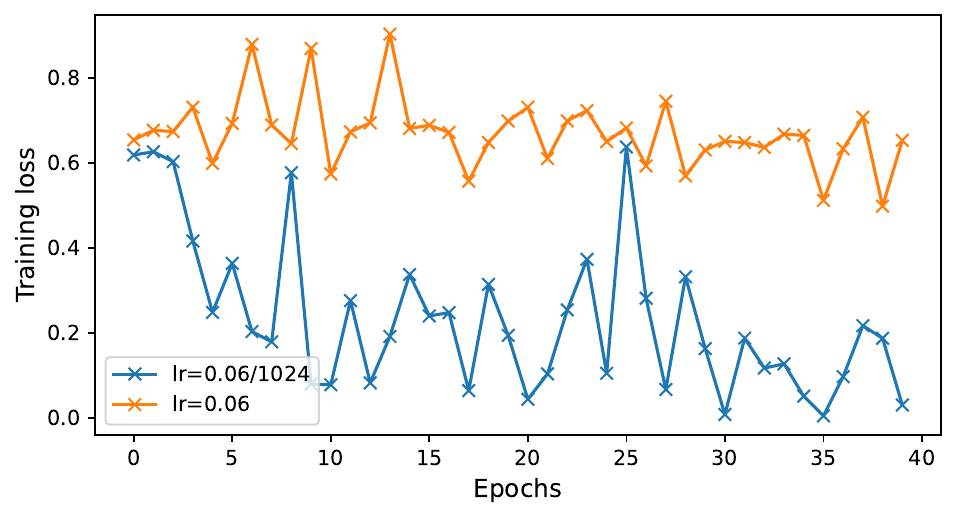}
}
\newline
\subfloat[\scriptsize Partition size $p=512$]{ 
\includegraphics[width=0.44\textwidth]{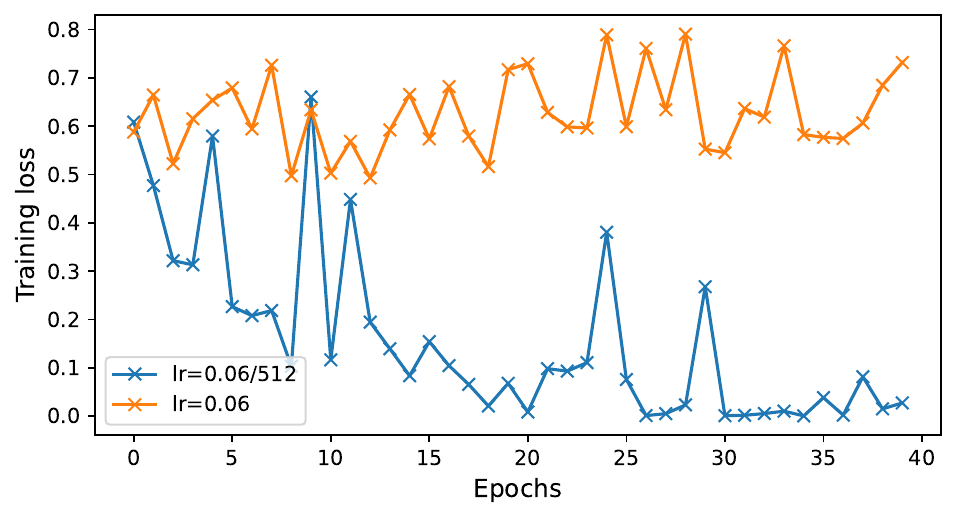}
}
\caption{
Training loss curve of the RoBERTa-large model on the MRPC dataset with our block circulant adapter. 
The learning rate 0.06 is the setting for training the FourierFT adapter that is also a Fourier domain based method like ours.
It can be seen that without applying our heuristic (Eq.\ref{eq:lr}), the model training diverges. However, after using the heuristic, the model successfully converges. }
\label{fig:convergence}
\end{figure}

\section{Experiments}
\label{sec:experiments}
The proposed block circulant adapter is evaluated by fine-tuning large language models on natural language processing tasks. 
To verify the effectiveness of our method, we analyze the downstream task performance, fine-tuning complexity and model convergence.  
The complexity analysis presents storage and computation complexity of the experimental model. 
The convergence analysis confirms the effectiveness of our proposed solution as expressed in Eq. (\ref{eq:lr}).

\textbf{Models and Datasets}. 
We fine-tune both the RoBERTa model \cite{liu2019roberta} and the LLaMA2-7B model \cite{llama}, which are the most frequently selected models for fine-tuning adapters. RoBERTa model has two different versions: RoBERTa-base with hidden dimension $n=768$ and RoBERTa-large with $n=1024$.

For RoBERTa models, we evaluate on the GLUE benchmark dataset, a standard multi-task dataset proposed by \cite{Wang2018GLUEAM} for natural language understanding.
Following \cite{DBLP:conf/icml/GaoWCLWC024}, we run experiments on Corpus of Linguistic Acceptability (CoLA) by \cite{CoLA}, Stanford Sentiment Treebank (SST-2) by \cite{SST-2}, Microsoft Research Paraphrase Corpus (MRPC) by \cite{MRPC}, Semantic Textual Similarity Benchmark (STS-B) by \cite{STS-B}, Question Natural Language Inference  (QNLI) by \cite{QNLI}, and Recognizing Textual Entailment (RTE) by \cite{RTE}. 

For the LLaMA2-7B model, we train on a cleaned version of Alpaca  \cite{alpaca} and evaluate on MT-Bench \cite{mtbench}, which contains 51K instruction-response pairs for instruction tuning. The cleaned version addresses issues like hallucinations, merged instructions, and empty outputs. We also evaluate on the GSM8K dataset \cite{cobbe2021training}, a high-quality dataset with 8.5K grad school math word problems.

\textbf{Baselines}.
We compare our proposed adapter with different adapters. The classical full fine-tuning (FF) adapter updates all parameters within the model.
Low rank adaption (LoRA) by \cite{LoRA} is a well-known fine-tuning method for large language models. 
It applies low rank factorization to weight change matrix and has been successfully used in various applications. 
Vector-based random matrix adaptation (VeRA) by \cite{vera} uses a pair of low-rank matrices shared among all layers and learns a small scaling vector. LaMDA++ by \cite{lamda} freezes the first projection matrix (PMA) in the adaptation path and the second projection matrix (PMB) in the early fine-tuning stage, while introducing a low-dimensional trainable square matrix. 
FourierFT \cite{DBLP:conf/icml/GaoWCLWC024} is the state-of-the-art fine-tuning method that trains with parameters in fourier domain.
Similarly, our method can also be seen as a fourier domain based method since FFT in Eq. (\ref{eq:circ_matvec}) directly transforms parameters to fourier domain. 
However, FourierFT uses 2D FFT computation rather than 1D FFT as in our method.

\textbf{Implementations}.
Our block circulant adapter is implemented using the PyTorch framework \cite{paszke2019pytorch}.
The partition size $p$ is set as large as possible to achieve the smallest storage and computation cost. 
Thus, $p$ is set as 768 and 256 for RoBERTa-base model, 1024 and 512 for RoBERTa-large model, 1024, 512, 256 and 128 for LLaMA2-7b model. 
Following \cite{DBLP:conf/icml/GaoWCLWC024}, we apply block circulant fine-tuning on query and value weight matrices inside the attention layer of two RoBERTa models and the LLaMA2-7B model fine-tuned on the alpaca dataset.
Following \cite{lamda}, we fine-tune on the MHSA and FFN layers of LLaMA2-7B model on the GSM8K dataset.
The classification head is fully fine-tuned.

\textbf{Metrics}.
Note that different datasets have different performance metrics. 
Matthew correlation coefficient (MCC) is report for CoLA, Pearson correlation coefficient (PCC) is reported for STS-B, and accuracy (Acc.) is reported for all remaining tasks. 
We evaluate the fine-tuned model on the Alpaca dataset using MT-Bench \cite{mtbench}, with GPT-4 \cite{gpt4} subsequently assigning scores to the model's responses for 80 multi-turn questions on a scale of 10. 
Following \cite{DBLP:conf/icml/GaoWCLWC024}, we perform 5 runs on each dataset with different random seeds and report the median metric value with standard deviation. 
Overall, larger task metric value means better model performance. 
For complexity analysis of all methods, we compute related number of trainable parameters and floating point operations (FLOPs). 
Smaller complexity value means better model performance. 
Letters K, M, G in experimental results indicate data volume units, and they stand for kilo, mega and giga, respectively.

\begin{figure}[!t]
\centering
\subfloat[ Adapter complexity on RoBERTa-base model.]{ 
\includegraphics[width=0.47\textwidth]{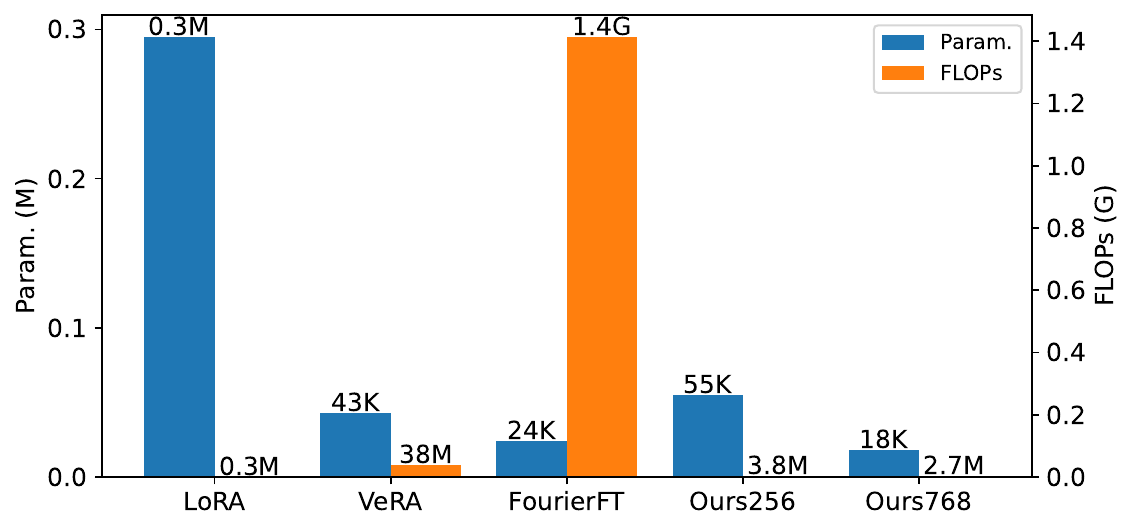}
}
\newline
\subfloat[ Adapter complexity on RoBERTa-large model.]{
\includegraphics[width=0.47\textwidth]{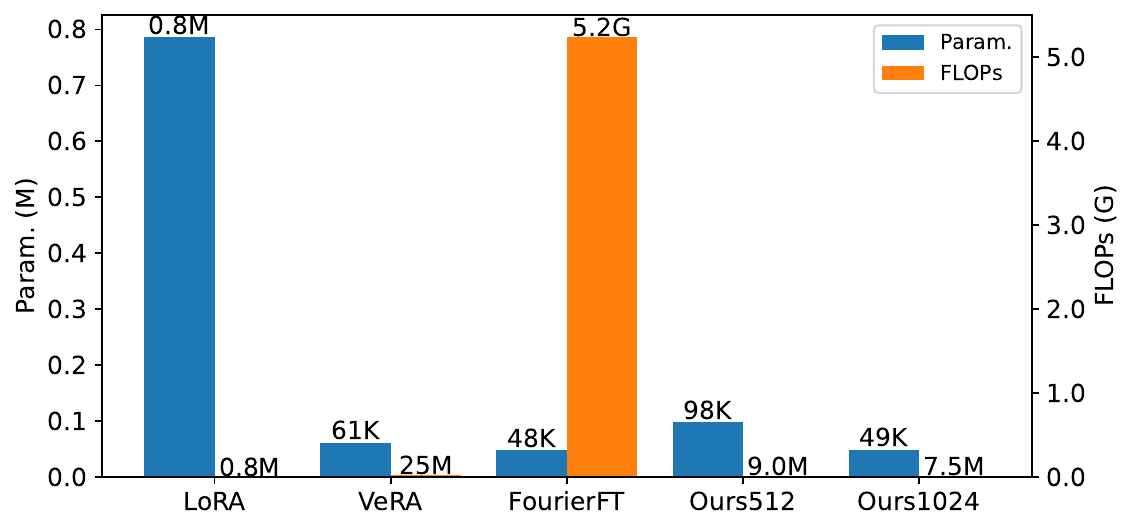}
}
\caption{
Complexity comparison of different adapters.
The larger partition size $p$ results in smaller storage and computation complexity of our block circulant adaper. 
FF method is not presented due to its high cost in both parameters and FLOPs. 
Our proposed block circulant adapters can balance between parameter amount and FLOPs. 
}
\label{fig:param_flops}
\end{figure}

\subsection{Convergence Analysis}
Given that our weight parameters are injective to its frequency domain representation, both FourierFT and our approach can be seen as Fourier domain based fine-tuning methods.
We adopt the same learning rate setting for fine-tuning on GLUE dataset.
Divergence of fine-tuning block circulant adapter on GLUE dataset can be observed on some tasks  without applying our proposed heuristic Eq. (\ref{eq:lr}). 
Such phenomenon can be caused by a bad initialization due to poorly chosen random seeds or the large gradients of block circulant matrix as proved in this paper.  
Bad initialization is an universally applicable hypothesis, which is also orthogonal to our proposal. 
However, according to our theoretical proof and empirical observation, we argue that our proposed heuristic can effectively ensure a stable training process. 

Fig. \ref{fig:convergence} shows an example of training loss curve comparison between with and without applying our proposed heuristic. 
It can be seen that when $p=1024$ and $p=512$, the training loss curve does not converge with learning rate 0.06 which is used for training FourierFT adapter. 
After adopting the proposed heuristic, it converges successfully.

\subsection{Fine-tuning Performance}
Table \ref{tbl:rob_glue} presents fine-tuned RoBERTa base and large models performance on GLUE dataset. 
In terms of RoBERTa base model, while $p=256$ results in more training parameters than $p=768$, our method with $p=256$ is slightly worse than our method with $p=768$ on average. 
This is mainly caused by the result on CoLA datasets, where over-parameterization seems not improving model performance. 
For rest datasets, $p=256$ achieves close or better performance than $p=768$.
For RoBERTa large model, our method with partition size $p=512$ achieves the best average performance. 
From $p=1024$ to $p=512$, our model performance improves significantly on average.
Note that partition size $p=1024$ is the maximum configurable partition size, since $n=1024$ for RoBERTa-large model. 
This shows even when $p=1024$ our method results in the smallest block circulant adapter, it still achieves competitive performance compared with other approaches. 

Table \ref{tbl:llama_mt_gsm8k} shows the performance of the fine-tuned LLaMA2-7b model on the Alpaca and GSM8K datasets. 
It can be noticed that increasing partition size results in decreasing parameter amount and FLOPs.
However, the task score or accuracy does not always decrease with larger partition size. 
This may be caused by the over-parameterization of the large LLaMA model or  the strong structure of the proposed method that may serve as a regularization. 
Compared with other adapters, our method can balance between parameters amount and FLOPs. 

\subsection{Complexity Analysis}
Our block circulant adapter has linear storage complexity because of the circulant structure. 
More specifically, the storage complexity is $O(n^2/p)$. 
Thus, larger partition size $p$ results in smaller number of parameters for fine-tuning. 
Our computation complexity is loglinear with the FFT operations for fast matrix vector multiplication, i.e., $O(\frac{n^2}{p}\log p)$. 
When $p$ is close to $n$, then the computation complexity becomes close to $O(n\log n)$.
In particular, when $p=n$, it is a single circulant matrix with exactly $O(n\log n)$ computation complexity. 

\begin{table}[tb]
\centering
\setlength{\tabcolsep}{3pt} 
\resizebox{0.46\textwidth}{!}{
\begin{tabular}{c | c c c|c c c}
\hline
\multirow{2}{*}{Method} & \multicolumn{3}{c|}{MT-Bench}  & \multicolumn{3}{c}{GSM8K} \\
~ & FLOPs & Param. & Score &  FLOPs & Param. & Acc.\\
\hline
LoRA & $0.03G$ & $33.55M$ & $5.20$ & $0.01G$ & $28.05M$ & $36.9$ \\
VeRA & $2.29G$ & $1.65M$ & $5.08$ & - & - & - \\
FourierFT & $133.14G$ & $0.06M$ & $5.18$ & - & - & -\\
LaMDA & - & - & - & $0.06G$ & $4.37M$ & $37.9$ \\
LaMDA++ & - & - & - & - & $5.12M$ & $38.2$ \\                 
$\text{Ours}_{p=128}$ & $0.32G$ & $8.39M$ & $5.38$  & $1.33G$ & $35.13M$ & $38.6$ \\
$\text{Ours}_{p=256}$  & $0.19G$ & $4.19M$ & $5.38$ & $0.79G$ & $17.56M$ & $39.7$ \\
$\text{Ours}_{p=512}$  & $0.12G$ & $2.10M$ & $5.26$ & $0.48G$ & $8.78M$ & $38.6$ \\
$\text{Ours}_{p=1024}$ & $0.08G$ & $1.05M$ & $5.38$ & $0.31G$ & $4.39M$ & $37.8$ \\
\hline
\end{tabular}
}
\caption{Fine-tuning performance of LLaMA2-7B model. Larger score or accuracy value indicates better model performance. Missing values ``-" means not available. For example, LaMDA++ does not specify rank for each module, thereby with unknown FLOPs.
}
\label{tbl:llama_mt_gsm8k}
\end{table}

Figure \ref{fig:param_flops} visualizes the complexity comparison between different adapters. 
LoRA has large storage complexity but small computation complexity due to its low rank structure design.
The storage complexity of VeRA is $O(n+r)$, and the computational complexity  is $O(nr)$. However, in practice VeRA method can take large $r$ to achieve good performance.
FourierFT has small storage complexity but large computation complexity because of its sparse parameter structure in Fourier domain and heavy 2D FFT operations. 
It can be seen that across all models, LoRA has the largest parameter amount, and FourierFT has the largest FLOPs.  
The complexity of the different models listed in Table \ref{tbl:llama_mt_gsm8k} verifies the analysis.

For our method, we choose $p$ as large as possible for achieving small storage and computation complexity. 

Our proposed method can result in a good balance between storage and computation complexity. 
More specifically, compared with FourierFT, ours FLOPs is $32\times$ smaller on RoBERTa-large and   RoBERTa-base, while maintaining close or smaller amount of parameters. 
When compared with LoRA, ours training parameter amount is $16\times$ smaller on RoBERTa-large and RoBERTa-base, but our FLOPs is $5\times$ more than LoRA. 
On RoBERTA-base, our method is much smaller than VeRA in terms of both parameter amount and FLOPs. 
For LLaMA2-7B model, ours also achieves significantly smaller FLOPs than FourierFT and smaller parameter amount than LoRA.
VeRA takes slightly more parameters than ours and costing much higher FLOPs. 
On GSM8K dataset, our method can achieve the smallest number of parameters while maintaining similar accuracy. 

\subsection{Overall}
The proposed block circulant adapter achieves similar or better downstream task performance when compared with other adapters. 
In terms of number of parameters, our method requires as small as the state-of-the-art FourierFT adapter, which is significantly smaller than classical LoRA adapter. 
Regarding FLOPs amount, ours is as small as LoRA, which is much more smaller than FourierFT and VeRA. 
In general, our method can balance between parameter amount and FLOPs while keeping similar or better performance.

\section{Conclusion}
In this paper, we present a noval adapter design based on block circulant matrix and study the potential divergence risk via theoretical proof and empirical experiments.
To achieve a stable training process, we propose a heuristic solution that can result in successful convergence. 
Experimental results also demonstrate that our block circulant adapter has both low storage and computation complexity while achieving competitive downstream task performance. 

\newpage

\appendix

\section*{Acknowledgments}

This work was ﬁnancially supported by the National Key R\&D Program of China (Grant No. 2024YFA1211400).

\section*{Contribution Statement}
Xinyu Ding and Meiqi Wang contributed equally to this work. 
Siyu Liao and Zhongfeng Wang are co-corresponding authors.

\bibliographystyle{named}
\bibliography{ijcai25}

\end{document}